\documentclass[oneside,onecolumn,letter]{article}
\usepackage[citestyle=numeric-comp,maxbibnames=8, sorting=none, giveninits=true]{biblatex}

\usepackage[hidelinks]{hyperref}       
\usepackage{url}            
\usepackage{booktabs}       
\usepackage{amsfonts}       
\usepackage{nicefrac}       
\usepackage{microtype}      
\usepackage{xcolor}         

\usepackage{graphicx}
\usepackage{amsmath}
\usepackage{amsthm}
\usepackage{placeins}
\usepackage{thmtools}
\usepackage{thm-restate}
\usepackage{amssymb}
\usepackage{enumitem}
\usepackage{algorithm2e}
\usepackage{wrapfig}
\usepackage{subcaption}
\usepackage{float}

\usepackage{authblk}
\usepackage[margin=1in]{geometry}
\usepackage[parfill]{parskip}

\makeatletter
\renewcommand\AB@affilsepx{, \protect\Affilfont}
\makeatother


\newcommand\blfootnote[1]{%
  \begingroup
  \renewcommand\thefootnote{}\footnote{#1}%
  \addtocounter{footnote}{-1}%
  \endgroup
}

\declaretheorem[name=Definition,numberwithin=section]{definition}
\declaretheorem[name=Lemma,numberwithin=section]{lemma}

\newtheorem*{remark}{Remark}

\newcommand{\expt}[2][]{\ensuremath{\mathbb{E}_{#1}\left[#2\right]}}
\newcommand{\uni}[3]{Uni(#1:#2 \backslash #3)}
\newcommand{\red}[2]{Red(#1:#2)}
\newcommand{\syn}[2]{Syn(#1:#2)}
\newcommand{\mi}[3][]{I_{#1}(#2;#3)}

\DeclareMathOperator*{\minimize}{minimize}

\title{Quantifying Knowledge Distillation Using\\Partial Information Decomposition}

\author[1]{Pasan Dissanayake$^*$}
\author[1]{Faisal Hamman}
\author[1]{Barproda Halder}
\author[2]{Ilia Sucholutsky}
\author[3]{\authorcr Qiuyi Zhang}
\author[1]{Sanghamitra Dutta}

\affil[1]{\emph{University of Maryland}}
\affil[2]{\emph{Princeton University}}
\affil[3]{\emph{Google Research}}
\date{}  
\renewcommand\Affilfont{\itshape\small}

\begin{document}

\maketitle

\begin{abstract}
Knowledge distillation deploys complex machine learning models in resource-constrained environments by training a smaller student model to emulate internal representations of a complex teacher model. However, the teacher's representations can also encode nuisance or additional information not relevant to the downstream task. Distilling such irrelevant information can actually impede the performance of a capacity-limited student model. This observation motivates our primary question: \emph{What are the information-theoretic limits of knowledge distillation?} To this end, we leverage Partial Information Decomposition to quantify and explain the transferred knowledge and knowledge left to distill for a downstream task. We theoretically demonstrate that the task-relevant transferred knowledge is succinctly captured by the measure of \emph{redundant} information about the task between the teacher and student. We propose a novel multi-level optimization to incorporate redundant information as a regularizer, leading to our framework of \emph{Redundant Information Distillation (RID)}. RID leads to more resilient and effective distillation under nuisance teachers as it succinctly quantifies task-relevant knowledge rather than \emph{simply aligning student and teacher representations.}

\end{abstract}

\blfootnote{\hspace{-0.5cm}Accepted at the 28\textsuperscript{th} International Conference on Artificial
  Intelligence and Statistics (AISTATS) 2025,  Mai Khao, Thailand. $^1$\{pasand, fhamman, bhalder, sanghamd\}@umd.edu \; $^2$ is2961@princeton.edu \; $^3$ qiuyiz@google.com}

\section{Introduction}
Modern-day machine learning requires large amounts of compute for both training and inference. Knowledge distillation~\cite{hintonDistillation, romeroFitnets} can be used to compress a complex machine learning model (the teacher) by distilling it into a relatively simpler model (the student). The term ``distillation'' in this context means obtaining some assistance from the teacher while training the student so that the student performs much better than when trained alone (see Figure \ref{fig_knowledgeDistillation}). In its earliest forms, knowledge distillation involved the student trying to match the output logits of the teacher~\cite{hintonDistillation}. More advanced methods focus on distilling multiple intermediate representations of the teacher to the corresponding layers of the student \cite{romeroFitnets, variationalKD, contrastiveRepDistillation, taskAware}. We also refer the reader to \cite{gouKDSurvey,sucholutsky2023getting} for surveys. 

Information theory has been instrumental in both designing \cite{variationalKD, contrastiveRepDistillation} and explaining \cite{zhangKnowledgePoints, wangDistillCheckpoints} knowledge distillation techniques. However, less attention has been given to characterizing the fundamental limits of the process from an information-theoretical perspective. Our goal is to bridge this gap by {\it first introducing new measures to quantify the ``transferred knowledge'' and ``knowledge to distill'' for a teacher and a student model given a target downstream task}. We bring in an emerging body of work called Partial Information Decomposition (PID) \cite{williamBeerPID, griffithIntersectionInfo, bertschingerUniqueInfo} to explain knowledge distillation. We define the knowledge to distill using the PID measure of ``unique'' information about the task that is available only with the teacher but not the student. As it follows, the transferred knowledge is succinctly quantified by the measure of ``redundant'' information that is common between the teacher and student. 

We propose a multi-level optimization that maximizes redundant information (transferred knowledge) as a regularizer for more effective distillation. While PID has been explored in a few avenues of machine learning, 
\begin{wrapfigure}[17]{r}{0.4\textwidth}
  \vspace{-0.5cm}
  \begin{center}  
  \includegraphics[width=0.3\textwidth]{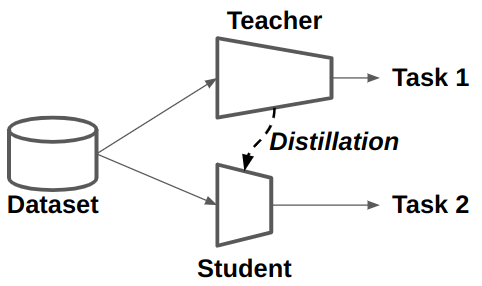}
  \end{center}
  \caption{Knowledge Distillation: The teacher (a complex model) assists the student (usually a substantially simpler model) during their training. The learned student can perform much better than an independently trained student without distillation with a similar training setup (i.e., hyper-parameters and data). The teacher may or may not have been trained for the same task as the student.
  \label{fig_knowledgeDistillation}}
\end{wrapfigure}
it has remained a challenge to maximize these measures as a regularizer since computing them itself requires solving an optimization. Our optimization leads to a novel knowledge distillation framework -- Redundant Information Distillation (RID) -- which precisely captures task-relevant knowledge and filters out the task-irrelevant information from the teacher. In summary, our main contributions are as follows:

\textbullet \; \textbf{Quantifying transferred knowledge and knowledge to distill:} Given a downstream task, and a teacher and a student model, we formally define the knowledge to distill as the unique information in the teacher (Definition~\ref{def_distillabledKnowledge}) and the transferred knowledge as the redundant information (Definition~\ref{def_distilledKnowledge}). Through examples and theoretical results (Theorem~\ref{thm_NuisanceTeacher}), we first show that redundant information succinctly captures task-related knowledge transferred to the student as opposed to the existing frameworks which directly align teacher ($T$) and student ($S$) representations, e.g., Variational Information Distillation (VID)~\cite{variationalKD} maximizes the mutual information $\mi{T}{S}$. Theorem~\ref{thm_NuisanceTeacher} points out a fundamental limitation of the existing knowledge distillation frameworks for capacity-limited students: they blindly align the student and the teacher without precisely capturing the task-related knowledge.

\textbullet \; \textbf{Maximizing redundant information as a regularizer:} To alleviate this limitation, we propose a strategy to incorporate redundant information as a regularizer during model distillation to precisely maximize the task-relevant transferred knowledge. We first circumvent the challenge of computing the redundant information measure proposed in \cite{bertschingerUniqueInfo} by utilizing the quantity termed intersection information defined in \cite{griffithRedInfo}, which we prove (in Theorem \ref{thm_griffithBrojaRed}) to be a lower-bound for redundant information. 
The significance of Theorem~\ref{thm_griffithBrojaRed} is that it enables us to obtain an optimization formulation to maximize a lower-bound of redundant information without making distributional assumptions, a contribution that is also of independent interest outside the domain of knowledge distillation. 

\textbullet \; \textbf{A novel knowledge distillation framework:} We propose a new framework called Redundant Information Distillation (RID) whose distillation loss is tailored to maximize the redundant information (i.e., the transferred knowledge as per Definition \ref{def_distilledKnowledge}). We carry out a number of experiments to demonstrate the advantage of this new framework over VID \cite{variationalKD}, an existing knowledge distillation framework that maximizes the mutual information between the teacher and the student representations (Section \ref{sec_experiments}). Experiments are carried out for the CIFAR10 and the CIFAR100 datasets, as well as for a transfer learning setting where the teacher is trained on the ImageNet dataset and transferred to the students over the CUB-200-2011 dataset. Our framework explains knowledge distillation and shows more resilience under less-informative nuisance teachers.

\textbf{Related Works:}
Multi-layer knowledge distillation was introduced in FitNets \cite{romeroFitnets}. Henceforth, a large number of techniques, based on different statistics derived for matching a teacher-student pair, have been proposed. In particular, \cite{variationalKD, contrastiveRepDistillation, chenWassersteinRepDistillation, milesInformationTheoreticRepDistillation} leverage an information-theoretic perspective to arrive at a solution (also see surveys~\cite{gouKDSurvey,sucholutsky2023getting}). In this paper, we focus on VID \cite{variationalKD} as a representative framework of the larger class of distillation frameworks which maximize $\mi{T}{S}$ as the distillation strategy. We also discuss Task-aware Layer-wise Distillation (TED) \cite{taskAware} as a framework that filters out task-related information. Specifically, \cite{taskAware} highlight the importance of distilling only the task-related information when there is a significant complexity gap between the teacher and the student. Towards this end, \cite{zhuUndistillableClasses} point out the existence of non-distillable classes due to the unmatched capacity of the student model. We discuss some more related works on knowledge distillation~\cite{kunduSkepticalStudent,parkStudentFriendlyTeacher,benItsAllInTheHead,jiaoTinybert,huangGeneric2SpecificDistill} in Appendix \ref{appdx_relatedWorks}.

Information theory has been instrumental towards explaining the success of knowledge distillation. \cite{wangDistillCheckpoints} utilize information bottleneck principles \cite{tishbyInformationBottleneck, tishbyInformationBottleneckDNN} to explain how a teacher model may assist the student in learning relevant features quickly. \cite{zhangKnowledgePoints} observe the training process as systematically discarding knowledge from the input. Accordingly, the distillation helps the student to quickly learn what information to discard. Despite these attempts, we observe that there exists a gap in characterizing the fundamental limits of knowledge distillation which we seek to address using the mathematical tool of PID.

PID is also beginning to generate interest in other areas of machine learning~\cite{dutta2020information,duttaFairnessFeatureExempt,duttaPIDinFairness,hamman2024demystifying,liangPIDCompute,liangMultimodalWithoutLabel,hamman2024isit,tax2017partial,ehrlich2022partial,wollstadt2023rigorous,mohamadi2023more,venkateshGaussianPIDBias,halder2024quantifying,dewan2024diffusion,venkateshGaussianPID,venkateshGaussianPIDBias,goswami2024analytically,rine,lyuExplicitPID}. However, it has not been leveraged in the context of knowledge distillation before. Additionally, while most related works predominantly focus on efficiently computing PID,~e.g.,~\cite{rine,liangPIDCompute,halder2024quantifying,pakman2021estimating} that itself requires solving an optimization over the joint distribution, \emph{there are limited works that further incorporate it as a regularizer during model training.} \cite{duttaFairnessFeatureExempt} leverage Gaussian assumptions to obtain closed-form expressions for the PID terms, enabling them to use unique information as a regularizer during training for fairness (also see \cite{venkateshGaussianPID,venkateshGaussianPIDBias} for more details on Gaussian PID). Our work makes novel connections between two notions of redundant information and demonstrates how PID can be integrated as a regularizer in a multi-level optimization without Gaussian assumptions, which could also be of independent interest outside the context of knowledge distillation.

\section{Preliminaries}
\textbf{Background on PID:} Partial Information Decomposition (PID), first introduced by \cite{williamBeerPID}, offers a way to decompose the joint information in two sources, say $T$ and $S$, about another random variable $Y$ 
\begin{wrapfigure}[12]{r}{0.4\textwidth}
  \vspace{-0.4cm}
  \centering
  \includegraphics[width=0.33\textwidth]{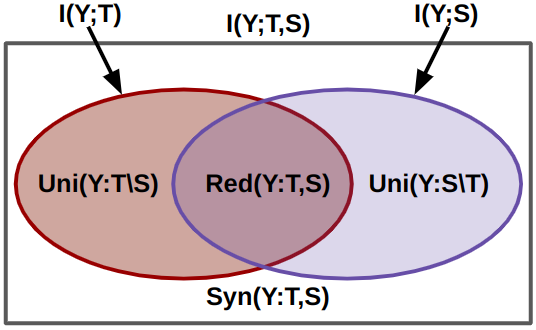}
  \caption{Partial Information Decomposition: The box denotes the total joint mutual information $\mi{Y}{T,S}$, which is decomposed into four non-negative terms named synergistic information $\syn{Y}{T,S}$, redundant information $\red{Y}{T,S}$ and the two unique information terms $\uni{Y}{T}{S}$ and $\uni{Y}{S}{T}$.}
  \label{fig_pid}
\end{wrapfigure}
(i.e., $\mi{Y}{T,S}$ where $\mi{A}{B}$ denotes the mutual information between $A$ and $B$ \cite{thomasCoverInfoTheory}) into four components as follows:
\begin{enumerate}[leftmargin=0.5cm, noitemsep, topsep=0pt]
    \item Unique information $\uni{Y}{T}{S}$ and $\uni{Y}{S}{T}$: information about $Y$ that each source uniquely contains
    \item Redundant information $\red{Y}{T,S}$: the information about $Y$ that both $T$ and $S$ share
    \item Synergistic information $\syn{Y}{T,S}$: the information about $Y$ that can be recovered only by using both $T$ and $S$. 
\end{enumerate}
See Figure \ref{fig_pid} for a graphical representation. These PID components satisfy the relationships given below:
\begin{align}
    \mi{Y}{T,S} &= \uni{Y}{T}{S} + \uni{Y}{S}{T} \nonumber\\
    &\qquad + \red{Y}{T,S} + \syn{Y}{T,S} \\
    \mi{Y}{T} &= \uni{Y}{T}{S} + \red{Y}{T,S} \\
    \mi{Y}{S} &= \uni{Y}{S}{T} + \red{Y}{T,S}.
\end{align}

While this system of equations cannot be solved to arrive at a deterministic definition for each PID term, defining only one of the terms is sufficient to define the rest. Consequently, a wide array of definitions exists, each based on different desired properties
\cite{williamBeerPID, bertschingerUniqueInfo, griffithIntersectionInfo, griffithRedInfo}. Among these, the definition proposed in \cite{bertschingerUniqueInfo} is motivated with an operational interpretation of unique information from decision theory. Moving on to the context of knowledge distillation, we map $T$ to be the teacher representation, $S$ to be the student representation, and $Y$ to be the downstream task that the student is being trained for. That makes $\mi{Y}{T}$ and $\mi{Y}{S}$ be the total knowledge about $Y$ that is in the teacher and in the student, respectively.

\textbf{Notation and Problem Setting:} We consider a layer-wise distillation scheme where the teacher representation $T(X)$ is distilled into the student representation $S_{\eta_s}(X)$, where $X$ is the input. The target of the student is to predict the task $Y$ from $X$. Both $T(\cdot)$ and $S_{\eta_s}(\cdot)$ are deterministic functions of $X$ and the randomness is due to the input being random. Note that the student representation depends on the parameters of the student network denoted by $\eta_s$ and hence written as $S_{\eta_s}$. However, when this parameterization and dependence on $X$ is irrelevant or obvious, we may omit both and simply write $T$ and $S$. We denote the supports of $Y,T$ and $S$ by $\mathcal{Y},\mathcal{T}$ and $\mathcal{S}$, respectively. In general, upper-case letters denote random variables, except $P$ and $Q$ which represent probability distributions, $C, H, W$ which stand for the representation dimensions, and $K$ which indicates the number of layers distilled. Lowercase letters are used for vectors unless specified otherwise. Lowercase Greek letters denote the parameters of neural networks. 

Knowledge distillation is achieved by modifying the student loss function to include a distillation loss term in addition to the ordinary task-related loss as follows:
\begin{equation}
    \mathcal{L}(\eta_s) = \lambda_1 \mathcal{L}_\text{ordinary}(Y,\hat{Y}(X)) + \lambda_2 \mathcal{L}_\text{distill}(Y, \hat{Y}(X), S_{\eta_s}, T).
\end{equation}
Here, $\lambda_1, \lambda_2 >0$. When the task at hand is a classification, $Y$ denotes the true class label and we use the cross entropy loss defined as $\mathcal{L}_{CE}( Y,\hat{Y})=-\expt[P_{X}]{\log P_{\hat{Y}(X)}(Y)}$ as the ordinary task-related loss for the student. Here, $\hat{Y}(X)$ is the student's final prediction of $Y$. The teacher network is assumed to remain unmodified during the distillation process.

\section{Explaining Knowledge Distillation}
\label{sec_quantifyingKD}
In this section, we propose information theoretic metrics to quantify both the task-relevant information that is available in the teacher for distillation, and the amount of information that has already been transferred to the student. We mathematically demonstrate favorable properties of our proposed measures in comparison to other candidate measures. Our mathematical results highlight the limitations of existing knowledge distillation frameworks that naively align the student with the teacher with no regard for task-relevance.

\begin{definition}[Knowledge to distill]
\label{def_distillabledKnowledge}
    Let $Y$, $S$, and $T$ be the target variable, the student's intermediate representation, and the teacher's intermediate representation, respectively. The knowledge to distill from $T$ to $S$ is defined as $\uni{Y}{T}{S}$, the unique information about $Y$ that is in $T$ but not in $S$.
\end{definition}
With the knowledge to distill is defined as the unique information $\uni{Y}{T}{S}$, we see that the more the distillation happens, the more the $\uni{Y}{T}{S}$ shrinks. Note that under the knowledge distillation setting, the total knowledge of the teacher $\mi{Y}{T}$ is constant since the teacher is not modified during the process. Since $\mi{Y}{T} = \uni{Y}{T}{S} + \red{Y}{T,S}$ we therefore propose $\red{Y}{T,S}$ as a measure for knowledge that has already been transferred.
\begin{definition}[Transferred knowledge]
\label{def_distilledKnowledge}
    Let $Y$, $S$, and $T$ be the target variable, student's intermediate representation, and the teacher's intermediate representation, respectively. The transferred knowledge from $T$ to $S$ is defined as $\red{Y}{T,S}$, the redundant information about $Y$ between $T$ and $S$.
\end{definition}
We leverage the unique and redundant information definitions given by \cite{bertschingerUniqueInfo} for an exact quantification of these quantities.
\begin{definition}[Unique and redundant information \cite{bertschingerUniqueInfo}]
\label{def_brojaRedUni}
Let $P$ be the joint distribution of $Y, T$ and $S$, and $\Delta$ be the set of all joint distributions over $\mathcal{Y}\times \mathcal{T}\times \mathcal{S}$. Then,
\begin{align}
    \uni{Y}{T}{S} &:= \min_{Q\in \Delta_P} \mi[Q]{Y}{T|S} \\
    \red{Y}{T,S} &:= \mi{Y}{T} - \min_{Q\in \Delta_P} \mi[Q]{Y}{T|S}
\end{align}
where $\Delta_P=\{Q\in \Delta:Q(Y=y,T=t)=P(Y=y,T=t), Q(Y=y,S=s)=P(Y=y,S=s) \; \forall \; y\in\mathcal{Y}, t\in\mathcal{T} \text{ and }s\in\mathcal{S}\}$ i.e., $\Delta_P$ is the set of all joint distributions with marginals of the pairs $(Y,T)$ and $(Y,S)$ equal to that of $P$.
\end{definition}

\textbf{Comparison to Existing Approaches for Knowledge Distillation:} A multitude of knowledge distillation frameworks exist which are based on maximizing the mutual information between the teacher and the student (i.e., $\mi{T}{S}$) \cite{variationalKD, contrastiveRepDistillation, chenWassersteinRepDistillation, milesInformationTheoreticRepDistillation}. While a distillation loss that maximizes $\mi{T}{S}$ can be helpful to the student when the teacher possesses task-related information, we show that it creates a tension with the ordinary loss when the teacher has little or no task-relevant information. Moreover, even though the teacher contains task-related information, the limited capacity of the student may hinder a proper distillation when this kind of framework is used. The following examples provide critical insights, exposing the limitation of $\mi{T}{S}$. Our proposed measure $\red{Y}{T,S}$ resolves these cases in an explainable manner by succinctly capturing task-relevant knowledge.


\textbf{Example 1:} (Uninformative teacher) 
An uninformative teacher representation (i.e., $T$ with $\mi{Y}{T}=0$) gives $\uni{Y}{T}{S}=\red{Y}{T,S}=0$ for any $S$, agreeing with the intuition. Hence, an algorithm that maximizes exactly the transferred knowledge $\red{Y}{T,S}$ will have a zero gradient over this term. In contrast, algorithms that maximize the similarity between $S$ and $T$ quantified by $\mi{T}{S}$ will force $S$ to mimic the uninformative teacher, causing a performance worse than ordinary training without distillation. As a simplified example, let $U_1, U_2 \sim Ber(0.5)$ and $Y=U_1, T=U_2$. Then, the teacher cannot predict the intended task $Y$. Note that, in this case, $\mi{T}{S}$ is not maximized when the student representation is $S=Y$. Instead, it is maximized when $S = U_2$.

\textbf{Example 2:} (Extra complex teacher) Let $U_1\sim Ber(0.2), U_2\sim Ber(0.5)$ and $Y=U_1, T=(U_1, U_2)$. Then, the teacher can completely predict the intended task $Y$. Assume the student is simpler than the teacher and has only one binary output. In this situation, $\mi{T}{S}$ is not maximized when $S=U_1$ because $\mi{(U_1, U_2)}{U_1}\approx 0.72 < 1 = \mi{(U_1, U_2)}{U_2}$ where the right-hand side is achieved when $S=U_2$. However, $S = U_1$ is a maximizer for $Red(Y:T,S)$ (i.e., $\red{Y}{T,S}=\red{U_1}{T,U_1}=\mi{Y}{T}$). Theorem~\ref{thm_NuisanceTeacher} presents a more general case.

Theorem~\ref{thm_NuisanceTeacher} formally exposes the limitations of maximizing $\mi{T}{S}$ for capacity-limited students as $\mi{T}{S}$ does not emphasize task-related knowledge.

\begin{restatable}[Teacher with nuisance]{theorem}{thmNuisanceTeacher}
    \label{thm_NuisanceTeacher}
    Let $T=(Z, G)$ where $Z$ contains all the task-related information (i.e., $\mi{Y}{T}=\mi{Y}{Z}$) and $G$ does not contain any information about the task (i.e., $\mi{Y}{G}=0$). Let the student be a capacity-limited model as defined by $H(S) \leq \max\{H(Z), H(G)\}$ where $H(X)$ denotes the Shannon entropy of the random variable $X$. Then,
    \begin{enumerate}[label=(\roman*)]
        \item $\mi{T}{S}$ is maximized when
        \begin{equation}
            S = \left\{\begin{matrix}
                Z &; & H(Z) > H(G) \\
                G &; & H(Z) < H(G)
            \end{matrix}\right..
        \end{equation}
        \item $\red{Y}{T,S}$ is always maximized when $S=Z$.
    \end{enumerate}
\end{restatable}
The uninformative random variable $G$ here can be seen as a stronger version of the nuisance defined in \cite[Section 2.2]{achilleNuisance}. In the above scenario, the task-related part of the student loss will have a tension with the distillation loss when $H(Z) < H(G)$, in which case, the distillation actually adversely affects the student. In contrast, a loss that maximizes $\red{Y}{T,S}$ will always be aligned with the task-related loss.

These examples show that the frameworks based on maximizing $\mi{T}{S}$ are not capable of selectively distilling the task-related information to the student. In an extreme case, they are not robust to being distilled from a corrupted teacher network. This is demonstrated in the experiments under Section \ref{sec_experiments}. 

It may appear that using information gain (conditional mutual information) $\mi{Y}{T|S}=\mi{Y}{T,S}-\mi{Y}{S}$ as the measure for the amount of knowledge available to distill resolves the cases similar to Example 1. However, Example 3 below provides a counterexample.

\textbf{Example 3:} (Effect of synergy) Consider a scenario similar to Example 1, where the teacher is uninformative regarding the interested task. For example, let $U_1, U_2 \sim Ber(0.5)$ and $Y=U_1, T=U_1\oplus U_2$ where $\oplus$ denotes the binary XOR operation. Suppose we were to consider conditional mutual information $\mi{Y}{T|S}$ as the measure of the amount of knowledge available to distill information available in the teacher. Then, $\mi{Y}{T|S}=H(Y)$ when $S=U_2$, indicating non-zero knowledge available to distill in the teacher. This is un-intuitive since in this case both $\mi{Y}{T}=\mi{Y}{S}=0$ and neither the teacher nor the student can be used alone to predict $Y$. In contrast, the proposed measures give $\uni{Y}{T}{S}=\red{Y}{T,S}=0$ indicating no available or transferred knowledge. 

Next, we present Theorem \ref{thm_properties} which highlights some important properties of the proposed metrics. These properties indicate that the proposed measures agree well with our intuition for explaining distillation.

\begin{restatable}[Properties]{theorem}{thmProperties}
\label{thm_properties}
The following properties hold for knowledge to distill and already transferred knowledge defined as in Definition \ref{def_distillabledKnowledge} and Definition \ref{def_distilledKnowledge} respectively.
\begin{enumerate}[labelindent=0pt,label=(\roman*)]
    \item $\uni{Y}{T}{S}$ and $\red{Y}{T,S}$ are non-negative.

    \item When $\uni{Y}{T}{S}=0$, the teacher has zero knowledge available for distillation. At this point, the student has the maximum information that any one of the representations $T$ or $S$ has about $Y$; i.e.,
    $\max \{\mi{Y}{T}, \mi{Y}{S}\} = \mi{Y}{S}$.
    
    \item For a given student representation $S$ and any two teacher representations $T_1$ and $T_2$ if there exists a deterministic mapping $h$ such that $T_1=h(T_2)$, then $\uni{Y}{T_1}{S} \leq \uni{Y}{T_2}{S}$.
\end{enumerate}
\end{restatable}

\section{A Framework For Maximizing Transferred Knowledge}
\label{sec_REDFrameworks}

\begin{figure*}[t]   
    \begin{subfigure}[b]{0.49\textwidth}
        \centering
        \includegraphics[width=\textwidth]{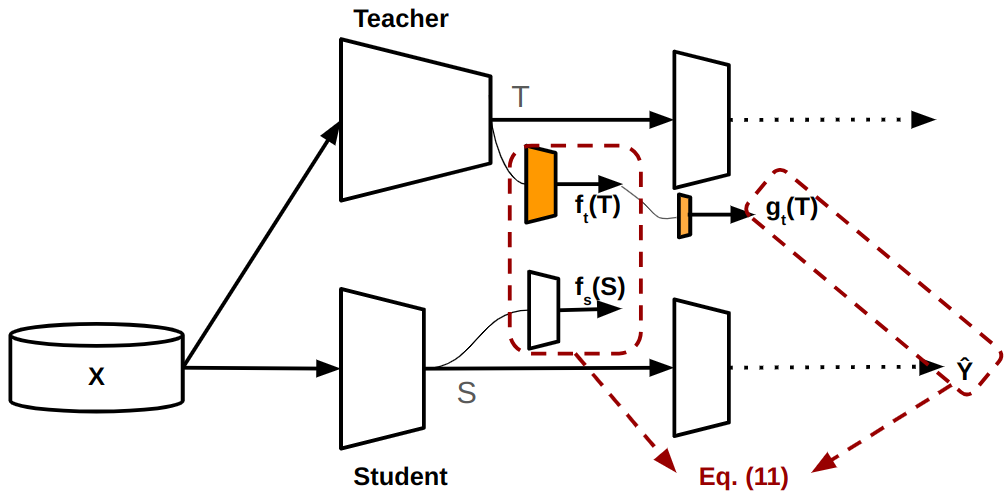}
        \caption{Phase 1: Train the teacher's filters while the student is kept frozen}
    \end{subfigure}
    \hspace{0.02\textwidth}
    \begin{subfigure}[b]{0.49\textwidth}
        \centering
        \includegraphics[width=\textwidth]{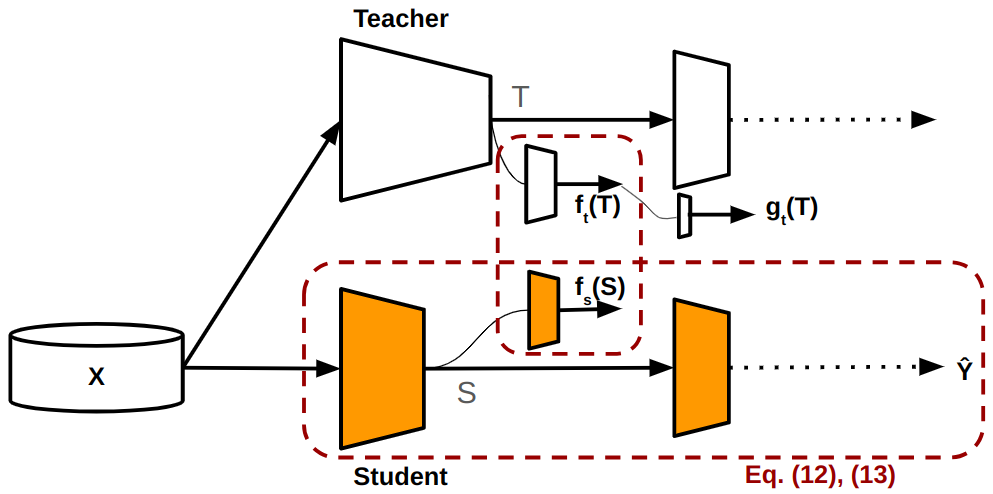}
        \caption{Phase 2: Train the student along with student's filters while the teacher is kept frozen}
    \end{subfigure}%
    \caption{Redundant Information Distillation (RID) framework: $f_t(\cdot)$ and $f_s(\cdot)$ are the teacher and the student filter outputs, respectively. A classification head $g_t(\cdot)$ is appended to the teacher's filter during the first phase. Highlighted in amber are the components that are being updated in each phase.
    \label{fig_RIDStructure}}
\end{figure*}

In this section, we propose a distillation framework -- Redundant Information Distillation (RID) -- which maximizes the transferred knowledge $\red{Y}{T,S}$ with a focus on classification problems. We first show that our measure of transferred knowledge is lower-bounded by an alternative definition of redundant information (also called the intersection information -- denoted by $Red_\cap (Y: T,S)$). Next, we leverage this lower-bound to develop a multi-level optimization framework to selectively distill task-relevant knowledge.

\begin{definition}[$I_\alpha$ measure \cite{griffithRedInfo}]
\label{def_griffithRed}
\begin{equation}
    Red_\cap (Y: T,S) = \max_{P(Q|Y)} \mi{Y}{Q} \quad \text{such that} \quad \mi{Y}{Q|f_t(T)}=\mi{Y}{Q|f_s(S)}=0.
\end{equation}
\end{definition}

Next we show, from Theorem \ref{thm_griffithBrojaRed}, that $Red_\cap (Y: T,S)$ is a lower-bound for $\red{Y}{T,S}$. The significance of this result is that maximizing the lower-bound $Red_\cap (Y: T,S)$ would also increase $\red{Y}{T,S}$.
\begin{restatable}[Transferred knowledge lower bound]{theorem}{thmGriffithBrojaRed}
    \label{thm_griffithBrojaRed}
    For any three random variables $Y, T$ and $S$,
    \begin{equation}
        Red_\cap (Y: T,S) \leq \red{Y}{T,S}
    \end{equation}
    where $Red_\cap (Y: T,S)$ is defined as per Definition \ref{def_griffithRed} and $\red{Y}{T,S}$ is defined in  Definition \ref{def_brojaRedUni}.
\end{restatable}

Next, we discuss our proposed approach for maximizing the lower-bound $Red_\cap (Y: T,S)$ (also see Figure~\ref{fig_RIDStructure} and Algorithm~\ref{algo_rid} for our complete strategy). Our proposed framework is based on selecting $Q$ in Definition \ref{def_griffithRed} to be $Q=f_t(T)$, and parameterizing $f_t(\cdot)$ and $f_s(\cdot)$ using small neural networks. To denote the parameterization, we will occasionally use the elaborated notation $f_t(\cdot;\theta_t)$ and $f_s(\cdot;\theta_s)$, where $\theta_t$ and $\theta_s$ denote the parameters of $f_t$ and $f_s$, respectively. With the substitution of $Q=f_t(T)$, Definition \ref{def_griffithRed} results in the optimization problem given below:
\begin{equation}
    \label{eq_basicRedOpt}
    \max_{\theta_t, \theta_s, \eta_s} \mi{Y}{f_t(T;\theta_t)} \quad \text{such that} \quad \mi{Y}{f_t(T;\theta_t)|f_s(S_{\eta_s};\theta_s)}=0. \tag{P1}
\end{equation}

We divide the problem \eqref{eq_basicRedOpt} into two phases and employ gradient descent on two carefully designed loss functions to perform the optimization. In the first phase, we maximize the objective w.r.t. $\theta_t$ while $\theta_s$ and $S$ are kept constant (recall that $T$ is fixed in all cases because the teacher is not being trained during the process). For this, we append an additional classification head $g_t(\cdot;\phi_t)$ parameterized by $\phi_t$ to the teacher's task aware filter $f_t$. We then minimize the following loss function with respect to $\theta_t$ and $\phi_t$:
\begin{equation}
\label{eq_ridTeacherLoss}
    \mathcal{L}_t(\theta_t, \phi_t) = \mathcal{L}_{CE}(Y, g_t(f_t(T;\theta_t);\phi_t)) + \sum_{c=1}^C\sum_{h=1}^H\sum_{w=1}^W \expt[P_{X}]{ \frac{V_{c,h,w}^2}{\sigma_c}}
\end{equation}
where $V_{c,h,w}$ denotes the corresponding element of $V=f_t(T(X);\theta_t) - f_s(S(X);\theta_s) \in\mathbb{R}^{C\times H\times W}$. Here, $C, H,$ and $W$ are the number of channels, height, and width of the outputs of $f_s$ and $f_t$. $\sigma=[\sigma_1,\dots,\sigma_C]^T$ is a stand-alone vector of weights that are optimized in the second phase. Minimizing the cross-entropy term $\mathcal{L}_{CE}(Y, g_t(f_t(T;\theta_t);\phi_t))$ of $\mathcal{L}_t(\theta_t, \phi_t)$ above amounts to maximizing $\mi{Y}{f_t(T;\theta_t)}$. The second term prohibits $f_t(T)$ from diverting too far from $f_s(S)$ during the process, so that the constraint $\mi{Y}{f_t(T;\theta_t)|f_s(S;\theta_s)}=0$ can be ensured.

During the second phase, we freeze $\theta_t$ and maximize the objective over $\theta_s, S_{\eta_s}$ and $\sigma$. The loss function employed in this phase is as follows:
\begin{equation}
    \mathcal{L}(\theta_s, \sigma, \eta_s) = \lambda_1 \mathcal{L}_{CE}(Y, \hat{Y}_{\eta_s}) + \lambda_2 \mathcal{L}_s(\theta_s, \sigma, \eta_s)
\end{equation}
where
\begin{equation}
\label{eq_ridStudentLoss}
\mathcal{L}_s(\theta_s, \sigma, \eta_s) = \Bigg( ||\sigma||^2 + \sum_{c=1}^C\sum_{h=1}^H\sum_{w=1}^W \expt[P_{X}]{ \frac{V_{c,h,w}^2}{\sigma_c}} \Bigg)
\end{equation}
and $\lambda_1$ and $\lambda_2$ are scalar hyperparameters which determine the prominence of ordinary learning (through $\mathcal{L}_{CE}(Y, \hat{Y}_{\eta_s})$) and distillation (through $\mathcal{L}_s(\theta_s, \sigma, \eta_s)$). $V$ and $\sigma$ are as defined earlier. $\hat{Y}_{\eta_s}$ denotes the final prediction of the student network. 

The first term of the loss function is the ordinary task-related loss. The next two terms correspond to the distillation loss, which is our focus in the following explanation. Consider phase 2 as an estimation problem that minimizes the $\sigma$-weighted mean squared error, where $Q=f_t(T)$ is the estimand and $f_s(\cdot)$ is the estimator. The magnitudes of the positive weights $\sigma$ are controlled using the term $||\sigma||^2$. We observe that this optimization ensures $\mi{Y}{Q|f_s(S)}=0$ given that the following assumption holds.

\textbf{Assumption:} {\it Let the estimation error be $\epsilon=f_t(T)-f_s(S)$. Assume $\mi{\epsilon}{Y|f_s(S)}=0$, i.e., given the estimate, the estimation error is independent of $Y$.}


With the above assumption, we see that
\begin{equation}
    \mi{Y}{Q|f_s(S)} = \mi{Y}{f_s(S)+\epsilon|f_s(S)} = \mi{Y}{\epsilon|f_s(S)}= 0.
\end{equation}
Therefore, the constraint in problem \ref{eq_basicRedOpt} is satisfied by this selection of random variables. Therefore, along with the maximization of $\mi{Y}{Q}$ during phase 1, the proposed framework can be seen as performing the optimization in Definition \ref{def_griffithRed} in two steps.

This, along with Theorem \ref{thm_griffithBrojaRed}, completes our claim that the proposed framework maximizes a lower bound for the transferred knowledge. RID loss terms  can be extended to multiple layers as follows: Replace the single term $\mathcal{L}_{t}(\theta_t, \phi_t)$ in \eqref{eq_ridTeacherLoss} with a sum of terms $\sum_{k=1}^K \mathcal{L}_t^{(k)}(\theta_t^{(k)}, \phi_t^{(k)})$ corresponding to each pair of layers $T^{(k)}$ and $S_{\eta_s}^{(k)}$ where$k=1,\dots,K$. Similarly, replace $\mathcal{L}_s(\theta_s,\sigma,\eta_s)$ in \eqref{eq_ridStudentLoss} with $\sum_{k=1}^K \mathcal{L}_s^{(k)}(\theta_s^{(k)}, \sigma^{(k)}, \eta_s)$. The framework is summarized in Algorithm \ref{algo_rid}. The advantage of RID over the VID framework \cite{variationalKD} which maximizes $\mi{T}{S}$ can be observed in the experiments that are detailed in Section \ref{sec_experiments}.
\begin{remark}
We note that the Large Language Model (LLM) fine-tuning framework of Task-aware Layer-wise Distillation (TED) \cite{taskAware} shares intuitive similarities with RID regarding distilling task-related knowledge. However, they take a heuristic approach when designing the framework. In fact, our mathematical formulation can explain the success of TED as detailed in Appendix \ref{appdx_vidAndTed}. In addition to the application domain, the difference between TED and RID can mainly be attributed to the following: 
\begin{enumerate}[label=(\roman*), wide, topsep=0pt, itemsep=0pt]
    \item During the first stage, TED trains both $f_t(\cdot)$ and $f_s(\cdot)$ whereas RID only trains $f_t(\cdot)$.
    \item In the second stage loss, TED includes an ordinary mean squared error term whereas RID includes a weighted (using $\sigma$) mean squared error term.
\end{enumerate}
To the best of our knowledge, our work is the first to information-theoretically quantify the actual task-relevant transferred knowledge and formally incorporate it into an optimization.
\end{remark}

\RestyleAlgo{ruled}
\SetKwComment{Comment}{/*}{*/}
\begin{algorithm}[th]
\caption{Redundant Information Distillation}\label{alg:two}
\KwData{A dataset of samples of (X,Y), teacher model with intermediate representations $T^{(1)},\dots,T^{(k)}$, hyperparameters $\lambda_1, \lambda_2 > 0$, \# warm-up epochs $n_w$, \# training epochs $n$, \# steps per cycle $q\leq n$, alternating ratio $r (0<r<1)$} 

\KwResult{Trained student network parameterized with $\eta_s$}

Initialize parameters $\theta^{(k)}_t, \theta^{(k)}_s, \phi^{(k)}_t$ and $\eta_s$\;

\For{$i \in \{1,\dots,n_w\}$}{
    $\displaystyle\minimize_{\{\theta_t^{(k)},\phi_t^{(k)}\}} \sum_{k=1}^K \mathcal{L}_{CE}(Y, g_t^{(k)}(f_t^{(k)}(T^{(k)}; \theta_t^{(k)}); \phi_t^{(k)}))$ \;
}

\For{$i \in \{1,\dots,n\}$}{
    \eIf{$\texttt{round}(i / q) < q\times r$}{
        $\displaystyle\minimize_{\{\theta_t^{(k)}, \phi_t^{(k)}\}} \sum_{k=1}^K  \mathcal{L}_t(\theta_t^{(k)}, \phi_t^{(k)})$ \Comment*[r]{See Eq. \eqref{eq_ridTeacherLoss}}
    }
    {
       $\displaystyle\minimize_{\{\theta_s^{(k)}, \sigma^{(k)}, \eta_s\}} \lambda_1 \mathcal{L}_{CE}(Y, \hat{Y}_{\eta_s}) + \lambda_2 \sum_{k=1}^K \mathcal{L}_s(\theta_s^{(k)}, \sigma^{(k)}, \eta_s)$ \Comment*[r]{See Eq. \eqref{eq_ridStudentLoss}}
    }
}
\label{algo_rid}
\end{algorithm}






\section{Empirical Validation}
\label{sec_experiments}

\textbf{Setup:} We compare the performance of the proposed RID framework with (i) VID; and (ii) our adaptation of TED in this domain under two different conditions. In the first setting, the teacher network is fully trained with the complete training set, whereas in the second setting, the teacher is just randomly initialized without any training at all. Experiments are carried out on CIFAR-10  \cite{cifar10Dataset}, CIFAR-100 \cite{cifar100Dataset}, ImageNet \cite{imageNetDataset} and CUB-200-2011 \cite{cub200Dataset} datasets (with the last two being used in a transfer learning setup). Additionally, we train a student without any knowledge distillation, which we label as ``BAS". Table \ref{tab_modelArchitectures} provides details of the model architectures used in each setting.

\begin{table}[h!]
    \centering
    \caption{Model architectures}
    \begin{tabular}{p{5cm}cc}
        \toprule 
         \textbf{Dataset} & \textbf{Teacher} & \textbf{Student} \\
         \midrule
         CIFAR-10 & WRN-(40,2) & WRN-(16,1) \\[5pt]
         CIFAR-100 & WRN-(28,10) & WRN-(16,8) \\[5pt]
         ImageNet to CUB-200-2011 & ResNet-34 & ResNet-18 \\
         \bottomrule
    \end{tabular}
    \label{tab_modelArchitectures}
\end{table}






We distill three layers from the teacher to the corresponding student layers. In the case of RID, each teacher layer $T^{(k)}$ has its own filter $f_t^{(k)}$ parameterized with $\theta_t^{(k)}$. Student filters are parameterized in a similar manner. Moreover, each teacher filter $f_t^{(k)}(\cdot)$ has its own classification head $g^{(k)}_t(\cdot)$ parameterized with $\phi^{(k)}$. All the student representations are parameterized by the complete weight vector $\eta_s$. In the beginning, the teacher filters are trained for $n_w$ number of warm-up epochs with just the cross-entropy loss $\sum_{k=1}^K \mathcal{L}_{CE}(Y, g_t^{(k)}(f_t^{(k)}(T^{(k)}; \theta_t^{(k)}); \phi_t^{(k)}))$. Then, the optimization alternates between the first and the second phases, with each cycle taking $q$ epochs in total. Within a cycle, phase 1 is carried out for $r\times q$ epochs followed by phase 2 for rest of the epochs (see Algorithm \ref{algo_rid}). Values of all the hyperparameters are given in Appendix \ref{appdx_experiments}. TED is implemented similarly, except the fact that now, the student filters also have a classification head $g_s^{(k)}(\cdot)$ (see Appendix \ref{appdx_vidAndTed}). Since TED acts as a fine-tuning method, we use the BAS student as the base model to apply TED on. 

In order to evaluate the validity of the Definitions \ref{def_distillabledKnowledge} and \ref{def_distilledKnowledge}, we compute the PID components \cite{bertschingerUniqueInfo} of the joint mutual information $\mi{Y}{T,S}$ of the innermost distilled layer using the estimation method given in \cite{liangPIDCompute}. See Appendix \ref{appdx_experiments} for additional details.

\textbf{Results:} Figure \ref{fig_classificationAcc} shows the classification accuracy on CIFAR-10 dataset for each student model RID, VID, and TED when distilled with either a trained teacher or an untrained one. It also shows the classification accuracies of the baseline model BAS (without any distillation) and the trained teacher. Figures \ref{fig_pidPlots} and \ref{fig_pidPlotsCifar100} present the PID values and the marginal mutual information estimates for the teacher and the student models for CIFAR-10 and CIFAR-100, respectively. Tables \ref{tab_cifar100Results} and \ref{tab_transferLearningResults} in Appendix \ref{appdx_experiments} present the classification accuracies for CIFAR-100 and the transfer learning setup ImageNet $\rightarrow$ CUB-200-2011, respectively. 

\begin{figure}[!ht]
    \centering
    \includegraphics[width=0.64\columnwidth]{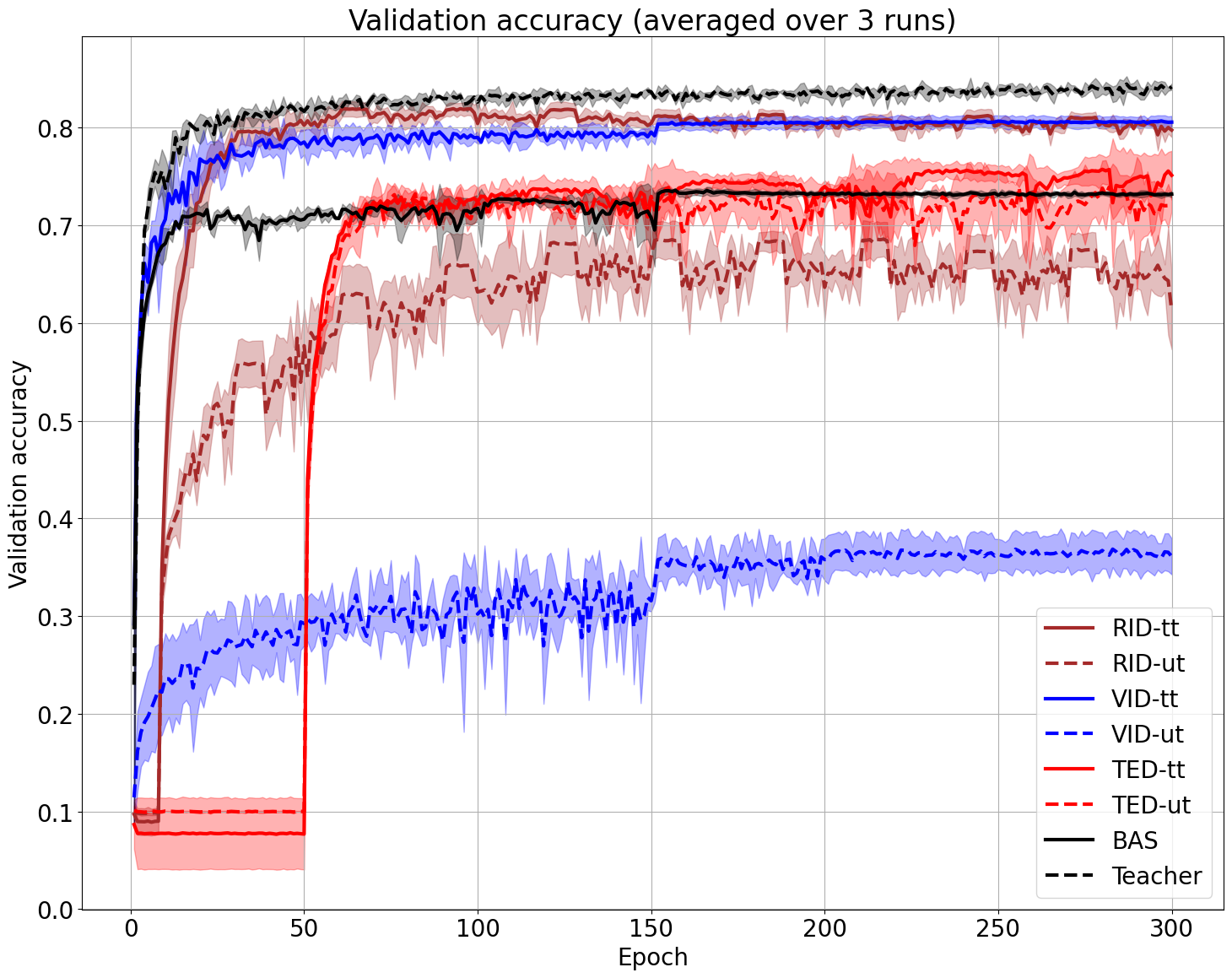}
    \caption{Classification accuracy for CIFAR-10 dataset of RID, VID, TED and BAS when distilled using a trained (abbreviated ``tt''--solid lines) and an untrained (abbreviated ``ut''--dashed lines) teacher: The solid and dashed lines indicate the mean over three runs. Shaded areas represent the corresponding confidence regions: mean $\pm$ standard deviation. Colors correspond to the distillation method used.}
    \label{fig_classificationAcc}
\end{figure}
\begin{figure*}[!ht] 
    \begin{subfigure}[b]{\textwidth}
    \centering   
    \includegraphics[width=0.8\textwidth]{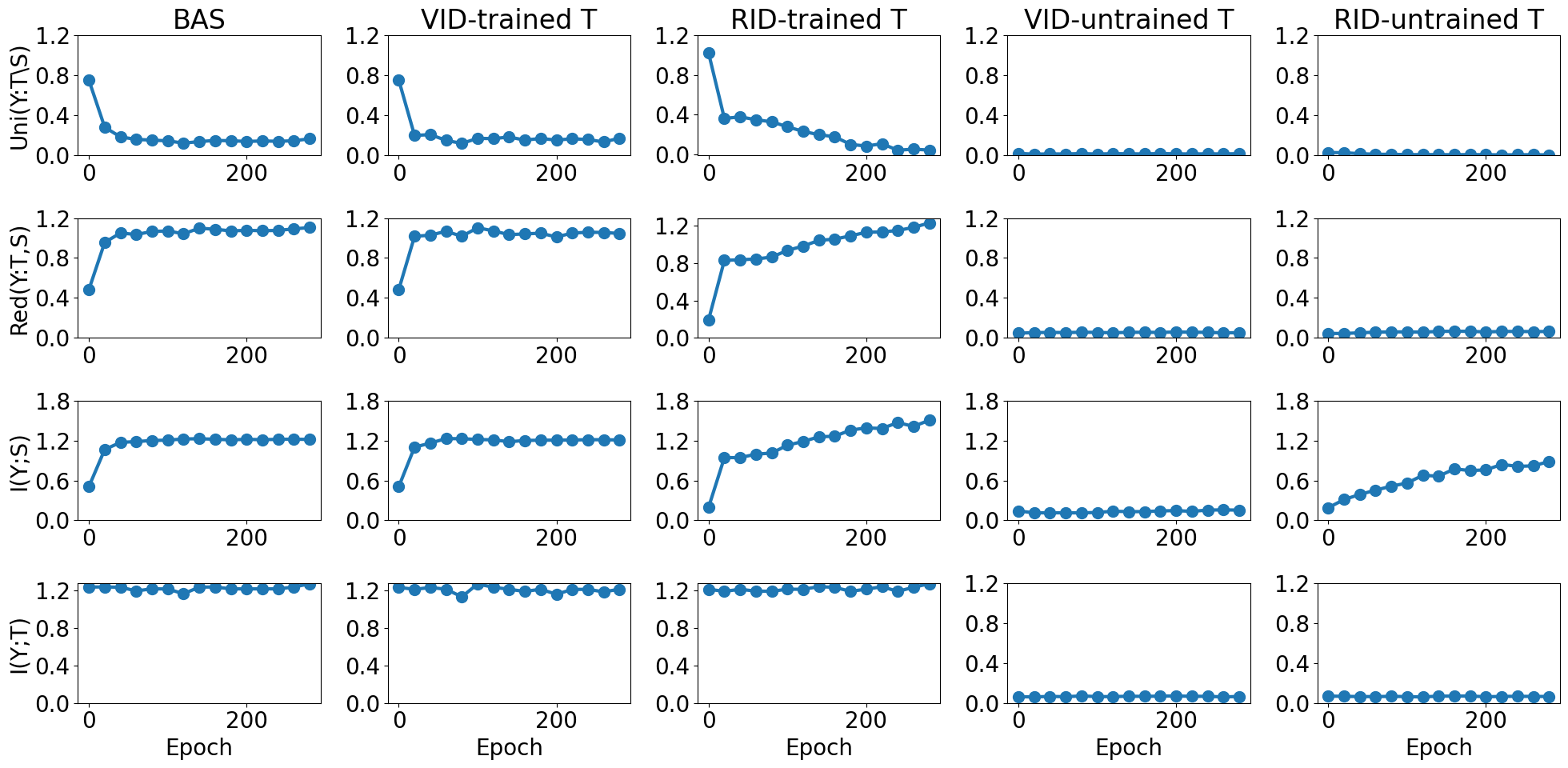}
    \caption{CIFAR-10}
    \label{fig_pidPlots}
    \end{subfigure}
    \begin{subfigure}[b]{\textwidth}
    \centering    
    \includegraphics[width=0.8\textwidth]{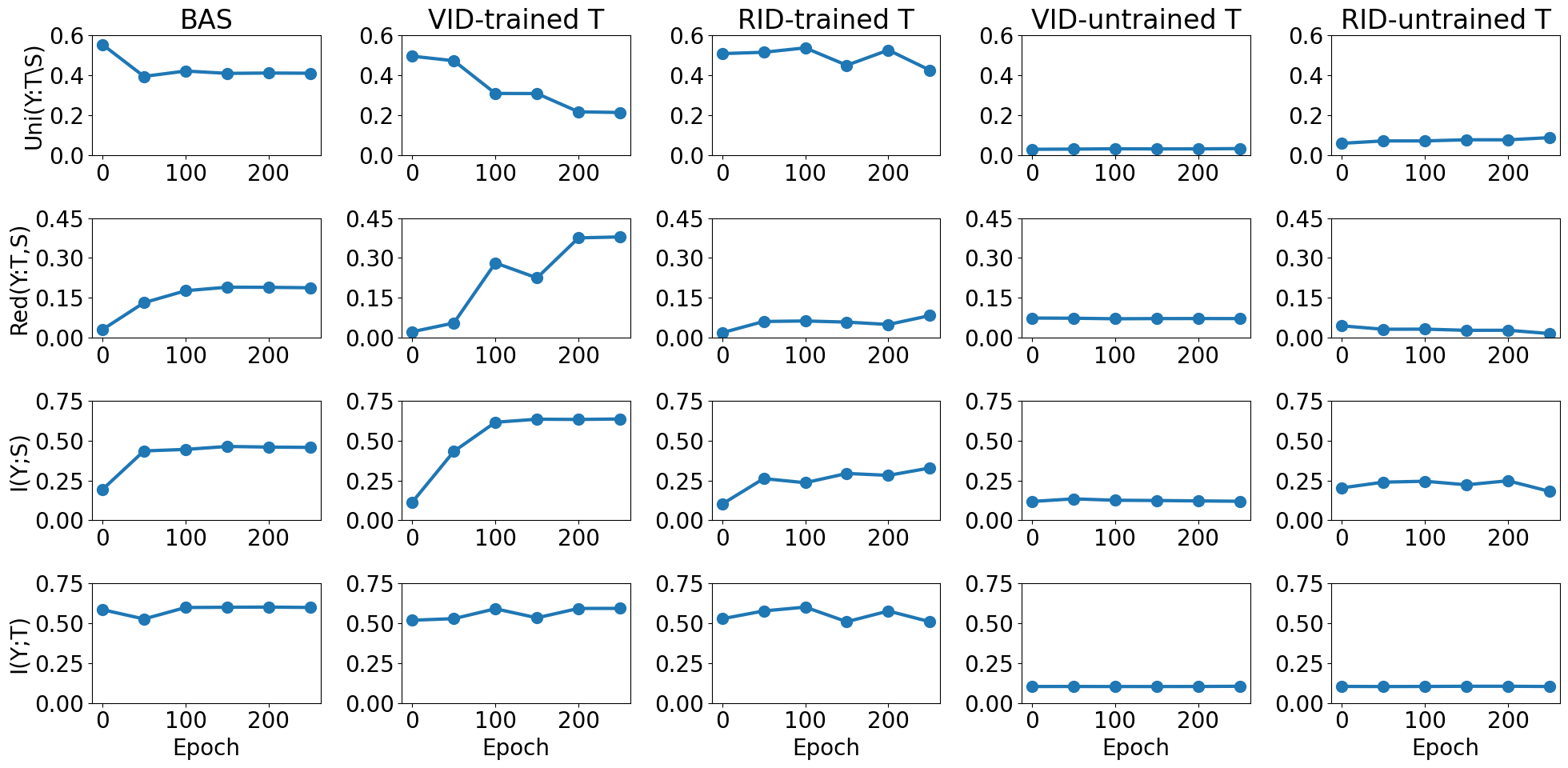}
    \caption{CIFAR-100}
    \label{fig_pidPlotsCifar100}
    \end{subfigure}
    \caption{Information atoms of $\mi{Y}{T,S}$ for BAS, VID and RID when distilled using a trained and an untrained teacher: Values are shown for the innermost distilled layer. The first two rows show that when distilled from a trained teacher, the remaining amount of knowledge available in the teacher for distillation $\uni{Y}{T}{S}$ decreases, whereas the already transferred knowledge $\red{Y}{T,S}$ increases. Observe from the third row how VID performs worse than both RID and BAS when the teacher is not trained.}
    \label{fig_pidPlotsCombined}
\end{figure*}

\textbf{Discussion:} From Figure \ref{fig_classificationAcc}, we see that both RID and VID perform equally well when the teacher is trained. However, the performance degradation of VID under the untrained teacher is significantly larger w.r.t. that of the RID. This can be attributed to the fact that the VID student trying to naively mimic the teacher. RID leads to more resilient and effective distillation under nuisance teachers as it succinctly quantifies task-relevant knowledge rather than simply aligning student and teacher representations. TED, despite being rather unstable, performs very close to the baseline model in both cases. As evident from Figure \ref{fig_pidPlotsCombined}, in the case of the trained teacher (whose $\mi{Y}{T}>0$), we observe that the knowledge available for distillation $\uni{Y}{T}{S}$ in the teacher decreases with the increasing number of epochs. Consequently, the amount of knowledge transferred $\red{Y}{T}{S}$ increases. When the teacher is not trained (i.e., $\mi{Y}{T}=0$), $\uni{Y}{T}{S},\red{Y}{T,S}\approx0$ as expected. Both BAS and RID models show an increase in $\mi{Y}{S}$ even when distilled from the untrained teacher. However, in this case, VID shows a very low $\mi{Y}{S}$ as expected, caused by the distillation loss forcing to mimic the teacher. A Python implementation of the experiments is available at \url{https://github.com/pasandissanayake/kd-rid}.

\section{Conclusion}

With the growing interest in knowledge distillation, our work provides critical insights into the explainability of knowledge distillation. We propose using $\uni{Y}{T}{S}$ to quantify the knowledge available in a teacher model for distillation w.r.t. a student and a downstream task. This, in turn, leads to the definition of the amount of knowledge that has been distilled to a student as $\red{Y}{T,S}$. We show that knowledge distillation frameworks which use mutual information between the teacher and the student representations to achieve knowledge distillation have a fundamental problem: These frameworks force the student to mimic the teacher regardless of the usefulness of the teacher's information to perform the task at hand. In contrast, through many examples we demonstrate that the proposed metrics can correctly characterize the amounts of knowledge available for distillation and the already transferred knowledge. Moreover, we show the advantage of the proposed metric by implementing a new distillation framework -- Redundant Information Distillation (RID) -- and comparing its performance with the existing technique VID \cite{variationalKD}. While VID and RID perform similarly when the teacher is well-trained for the downstream task, VID performance degrades largely when the teacher is not trained but RID performs close to a student model trained without distillation.

\textbf{Limitations and future work:} While the RID framework uses an alternative definition for redundant information, computation of exact $\red{Y}{T,S}$ during training can be computationally prohibitive due to the optimization over $\Delta_P$. Moreover, characterizing the extent to which the assumption in Section \ref{sec_REDFrameworks} holds is not explored in this work. Extending the mathematical formulation in Section \ref{sec_REDFrameworks} to analyze other knowledge distillation frameworks is an interesting path for future research. Other potential research directions include: (i) distilling from an ensemble of teachers~\cite{malininDistillingEnsembles} in a way that the adverse effects of corrupted teachers are mitigated; (ii) dataset distillation~\cite{ilia_dataset_distillation}; or (iii) distillation for model reconstruction from counterfactual explanations~\cite{dissanayake2024model}. Incorporating fundamentally different definitions for PID components, such as \cite{lyuExplicitPID} which provides explicit formulae, as regularizers can also be interesting.

\textbf{Acknowledgments:}  This work was supported in part by NSF CAREER Award 2340006 and Northrop Grumman Seed Grant. 

\FloatBarrier
\renewcommand\bibname{\large References}

\printbibliography

\clearpage
\newpage
\appendix
\section{Additional Related Works}
\label{appdx_relatedWorks}
The nature of the teacher model plays an important role in distillation. \cite{kunduSkepticalStudent} present a distillation scheme to distill from a teacher which has been intentionally made difficult-to-distill. They distill the penultimate layer of the teacher to several intermediate representations of the student through auxiliary student branches.  \cite{parkStudentFriendlyTeacher} prepare the teacher during its training phase so that it can later be used for better distillation when training a student. This is done through training an augmented version of the teacher formed by adding several student branches, all of which are trained together. \cite{benItsAllInTheHead} propose two novel frameworks: Teacher-Head Knowledge Distillation (TH-KD) and Student-Head Knowledge Distillation (SH-KD). In TH-KD, the teacher’s classification head is attached to the student and the distillation loss is computed as a weighted sum of the discrepancies of predictions from the two heads of the student with respect to the teacher’s predictions. SH-KD involves three steps: (i) a student is trained in the conventional way; (ii) the classification head of the student is fixed to the teacher and the teacher is trained while keeping the classification head frozen; and (iii) this teacher model is used to train a new student using conventional distillation. The assumption is that the teacher trained with the student classification head will adapt the teacher to transfer the knowledge better suited to the capacity-limited student. \cite{jiaoTinybert} propose an interesting knowledge distillation framework for compressing the well-known BERT models into a much smaller version called TinyBERT. The process consists of two steps. During the first step, a generic TinyBERT is distilled from a pre-trained BERT model using a large corpus. In the second step, this TinyBERT model is further fine-tuned by distilling a fine-tuned BERT on a task-specific dataset. While the broad goal of \cite{jiaoTinybert} seems to align with ours, the approach is quite different: we intend to filter out task-specific information from a generalized teacher by defining a measure that precisely captures this whereas \cite{jiaoTinybert} focus on distilling from a task-specific (i.e., fine-tuned) teacher in an efficient manner. The work on distilling vision transformers by \cite{huangGeneric2SpecificDistill} also proceeds in two steps. Similar to \cite{jiaoTinybert}, the first step focuses on distilling task-agnostic information using a pre-trained teacher within an encoder-decoder setup. In the second step, the decoder is abandoned and the task-specific information is transferred using a fine-tuned teacher. As pointed out in \cite[Section 4.3]{huangGeneric2SpecificDistill}, this framework can be seen as maximizing the mutual information (conditioned on the dataset being used) between the teacher and the student in each step (rather than quantifying task-specific information), and hence, can be conceptually categorized together with VID. In contrast, our RID framework focuses on quantifying and extracting task-specific information from a generalized teacher.

\section{Proofs}
\label{appdx_proofs}

\subsection{Proof of Theorem \ref{thm_NuisanceTeacher}}
\thmNuisanceTeacher*
\begin{proof}
To prove claim 1, observe that 
\begin{align}
    \mi{T}{S} &= H(T)- H(T|S) \\
    &= H(Z, G) - H(Z,G|S) \\
    &= H(Z) + H(G) - H(Z,G|S).
\end{align}
Now, $S=Z \implies H(Z, G|S)=H(G)$ and $S=G \implies H(Z, G|S)=H(Z)$. Therefore,
\begin{equation}
    \mi{T}{S} = \left\{\begin{matrix}
        H(Z) &; & S=Z \\
        H(G) &; & S=G
    \end{matrix}\right..
\end{equation}
Claim 1 follows from the above since $\mi{T}{S} \leq H(S) \leq \max\{H(Z),H(G)\}$.

To prove claim 2, first observe that $\mi{Y}{T}=\mi{Y}{Z} \implies \mi{Y}{G|Z}=0$. Now consider the conditional mutual information $\mi{Y}{T|S}$:
\begin{align}
    \mi{Y}{T|S} &= \mi{Y}{Z,G|S} \\
    &= \mi{Y}{G|S} + \mi{Y}{Z|G,S}
\end{align}
Note that the right-hand side above vanishes when $S=Z$. Therefore, $S=Z\implies\mi{Y}{T|S}=0$. Now since
\begin{equation}
    \red{Y}{T,S} = \mi{Y}{T} - \underbrace{\min_{Q\in \Delta_P} \mi[Q]{Y}{T|S}}_{=0 \text{ with } Q=P \text{ when }S=Z}
\end{equation}
and $\red{Y}{T,S} \leq \mi{Y}{T}$, setting $S=Z$ achieves the maximum $\red{Y}{T,S}$.
\end{proof}

\subsection{Proof of Theorem \ref{thm_properties}}
\label{appdx_proofTheoremProperties}
\thmProperties*
Proof of the second property is given below:
\begin{proof}
    \begin{align}
        \max & \{\mi{Y}{T}, \mi{Y}{S}\} \nonumber\\
        &= \max \Big\{\red{Y}{T,S}+\underbrace{\uni{Y}{T}{S}}_{=0}, \red{Y}{T,S}+\uni{Y}{S}{T} \Big\} \\
        &= \red{Y}{T,S}+\uni{Y}{S}{T}\\
        &= \mi{Y}{S}.
    \end{align}
\end{proof}
The first and the third properties directly follow from  \cite[Lemma 5]{bertschingerUniqueInfo} and \cite[Lemma 31]{uniqueDeficiencies}.

\subsection{Proof of Lemma \ref{lem_conditionalMIFunctions}}
\begin{lemma}
    \label{lem_conditionalMIFunctions}
    Let $Y, T$ and $S$ be any three random variables with supports $\mathcal{Y}, \mathcal{T}$ and $\mathcal{S}$ respectively and $g(\cdot)$ be a deterministic function with domain $\mathcal{S}$. Then
    \begin{equation}
        \mi{Y}{T|g(S), S} = \mi{Y}{T|S}.
    \end{equation}
\end{lemma}
\begin{proof}
    By applying the mutual information chain rule to $\mi{Y}{T, S, g(S)}$ we get
    \begin{align}
        \mi{Y}{T,& S, g(S)} \nonumber\\
        &= \mi{Y}{S} + \mi{Y}{T|S} + \mi{Y}{g(S)|T,S} \\
        &= \mi{Y}{S} + \mi{Y}{T|S}  + \underbrace{H(g(S)|T,S)}_{=0} - \underbrace{H(g(S)|Y,T,S)}_{=0}\\
        &= \mi{Y}{S} + \mi{Y}{T|S}.
    \end{align}
    Also, from a different decomposition, we get
    \begin{align}
        \mi{Y}{T, S, g(S)} &= \mi{Y}{S} + \underbrace{\mi{Y}{g(S)|S}}_{=0} + \mi{Y}{T|g(S),S} \\
        &= \mi{Y}{S} + \mi{Y}{T|g(S),S}.
    \end{align}
    Combining the two right-hand sides yields the final result.
\end{proof}

\subsection{Proof of Theorem \ref{thm_griffithBrojaRed}}
\thmGriffithBrojaRed*
\begin{proof}
    For a given set of random variables $Y, T$ and $S$, let $f_t^*(T)$ and $f_s^*(S)$ achieve the maximum $\mi{Y}{Q}$ in Definition \ref{def_griffithRed}, i.e., $Red_\cap (Y: T,S)=\mi{Y}{f_t^*(T)}$ while $\mi{Y}{f_t^*(T)|f_s^*(S)}=0$. We first observe that $\red{Y}{f_t^*(T), f_s^*(S)}=Red_\cap (Y: T,S)=\mi{Y}{f_t^*(T)}$ as shown below:
    \begin{align}
        \red{&Y}{f_t^*(T), f_s^*(S)} \nonumber\\
        &= \mi{Y}{f_t^*(T)}-\min_{Q\in\Delta_P} \mi[Q]{Y}{f_t^*|f_s^*(S)} \\
        &= \mi{Y}{f_t^*(T)} \quad (\because \mi{Y}{f_t^*(T)|f_s^*(S)}=0) \\
        &= Red_\cap (Y: T,S).
    \end{align}
    Next, we show that $\red{Y}{f_t^*(T), f_s^*(S)} < \red{Y}{T,S}$. In this regard, we use the following lemma due to \cite{bertschingerUniqueInfo}.

    \begin{restatable}[Lemma 25, \cite{bertschingerUniqueInfo}]{lemma}{lemBrojaUniMonotonicity}
        \label{lem_brojaUniMonotonicity}
        Let $X, Y, Z_1, Z_2\dots,Z_k$ and $Z_{k+1}$ be a set of random variables. Then, 
        \begin{equation}
            \uni{X}{Y}{Z_1, Z_2\dots,Z_k} \geq \uni{X}{Y}{Z_1, Z_2\dots,Z_k,Z_{k+1}}.
        \end{equation}
    \end{restatable}
    Consider the set of random variable $Y, f_t^*(T), f_s^*(S)$ and $S$. From the above lemma we get 
    \begin{align}
        \uni{Y&}{f_t^*(T)}{f_s^*(S)} \nonumber \\
        &\geq \uni{Y}{f_t^*(T)}{f_s^*(S), S} \\
        &= \mi{Y}{f_t^*(T)} - \mi[Q^*]{Y}{f_t^*(T)|f_s^*(S), S}
    \end{align}
    where $Q^*=\arg\min_{Q\in\Delta_P} \mi[Q^*]{Y}{f_t^*(T)|f_s^*(S), S}$. Now, by applying Lemma \ref{lem_conditionalMIFunctions} to the right-hand side we arrive at
    \begin{align}
        \uni{Y&}{f_t^*(T)}{f_s^*(S)} \nonumber\\
        &\geq \mi{Y}{f_t^*(T)} - \mi[Q^*]{Y}{f_t^*(T)|S}\\
        &= \uni{Y}{f_t^*(T)}{S}.
    \end{align}
    Next, observe that the following line arguments holds from Definition \ref{def_brojaRedUni}:
    \begin{align}
        &\uni{Y}{f_t^*(T)}{f_s^*(S)} \geq \uni{Y}{f_t^*(T)}{S} \\
        \iff & \mi{Y}{f_t^*(T)} - \uni{Y}{f_t^*(T)}{f_s^*(S)} \leq \mi{Y}{f_t^*(T)} - \uni{Y}{f_t^*(T)}{S}\\
        \iff &\red{Y}{f_t^*(T), f_s^*(S)} \leq \red{Y}{f_t^*(T), S}.
    \end{align}
    Noting that $\red{Y}{A,B}$ is symmetric w.r.t. $A$ and $B$, we may apply the previous argument to the pair $\red{Y}{f_t^*(T), S}$ and $\red{Y}{T, S}$ to obtain
    \begin{equation}
        \red{Y}{f_t^*(T), f_s^*(S)} \leq \red{Y}{f_t^*(T), S} \leq \red{Y}{T,S},
    \end{equation}
    concluding the proof.
\end{proof}

\section{VID And TED Frameworks}
\label{appdx_vidAndTed}
\subsection{Variational Information Distillation (VID)}
The VID framework \cite{variationalKD} is based on maximizing a variational lower bound to the mutual information $\mi{T}{S}$. It finds a student representation $S$ which minimizes the following loss function:
\begin{equation}
    \label{eq_vidLoss}
    \mathcal{L}_{VID}(\eta_s, \mu) = \mathcal{L}_{CE}(Y, \hat{Y}_{\eta_s}) + \lambda \sum_{c=1}^C \sum_{h=1}^H \sum_{w=1}^W \Bigg( \log \sigma_c + \expt[P_X]{\frac{(T_{c,h,w}-\mu_{c,h,w}(S_{\eta_s}))^2}{2\sigma_c^2}}\Bigg).
\end{equation}
Here, $C, H$ and $W$ are the number of channels, height and width of the representation $T$ respectively (i.e., $T\in\mathbb{R}^{C\times H\times W}$). $\mu$ is a deterministic function parameterized using a neural network and learned during the training process. $\sigma=[\sigma_1,\dots,\sigma_c]^T$ is a vector of independent positive parameters, which is also learned during the training process. $\hat{Y}_{\eta_s}$ is the final prediction of the student model of the target label $Y$.

\subsection{Task-aware Layer-wise Distillation (TED)}
The TED framework \cite{taskAware} fine-tunes a student in two stages. During the first stage, task-aware filters appended to the teacher and the student are trained with task-related heads while the student and the teacher parameters are kept constant. In the next step, the task-related heads are removed from the filters and the student is trained along with its task-aware filter while the teacher and its task-aware filter is kept unchanged. We observe that each of these steps implicitly maximizes the redundant information under Definition \ref{def_griffithRed}. To see the relationship between the TED framework and  the above definition of redundant information, let $Q$ be parameterized using the teacher's task-aware filter as $Q=f_t(T)$. Now consider the first stage loss corresponding to the teacher's task-aware filter which is given below:
\begin{equation}
    \label{eq_taskAwareStage1}
    \mathcal{L}_t\left(T,\theta_t\right) = \expt[x\sim\mathcal{X}]{\ell(f_t(T;\theta_t))}.
\end{equation}
Here, $\ell(\cdot)$ is the task specific loss, $f_t$ is the task-aware filter parameterized by $\theta_t$. During the first stage, this loss is minimized over $\theta_t$. A similar loss corresponding to the student (i.e., $\expt[x\sim\mathcal{X}]{\ell(f_s(S;\theta_t))}$) is minimized in order to train the student's task aware filter. Note that during this process, both $\mi{Y}{f_t(T)}$ and $\mi{Y}{f_s(S)}$ are increased.

During stage 2, the distillation loss which is given below is minimized over $\theta_s$ and $S$ while $\theta_t$ and $T$ being held constant.
\begin{equation}
    \label{eq_taskAwareStage2}
    \mathcal{D}_{TED}\left(T,S\right) = \expt[x\sim\mathcal{X}]{||f_t(T;\theta_t)-f_s(S;\theta_s)||^2}.
\end{equation}
Consider stage 2 as an estimation problem which minimizes the mean square error, where $Q=f_t(T)$ is the estimand and $f_s(\cdot)$ is the estimator. We observe that this optimization ensures $\mi{Y}{Q|f_s(S)}=0$ given that the same assumption as in Section \ref{sec_REDFrameworks} holds. Following similar steps as in Section \ref{sec_REDFrameworks}, we see that TED framework maximizes a lower bound for the transferred knowledge, quantified as in Definition \ref{def_distilledKnowledge}.

The main difference of this scheme w.r.t. the RID framework is two-fold. First, in RID we optimize $f_t(\cdot)$ in addition to $f_s(\cdot)$ and $S$ during stage 2. In contrast, TED does not modify the teacher's filter during the second stage. Second, RID distillation loss employs a weighting parameter similar to that of VID.

\section{Empirical Validation}
\label{appdx_experiments}
\subsection{Datasets} 
We use the CIFAR-10 dataset \cite{cifar10Dataset} with 60,000 32x32 colour images belonging to 10 classes, with 6,000 images per class. The training set consists of 50,000 images (5,000 per class) and the test set is 10,000 images (1,000 per class). The PID values are evaluated over the same test set. Similarly, CIFAR-100 dataset \cite{cifar100Dataset} contains 32x32 color images belonging to 100 classes with 600 images from each class. The dataset is split into training and test sets with 500 and 100 images per class in each split, respectively. While this split was used to train the teacher, the students were trained only on a subset of training data with only 100 samples per class. For the transfer learning setup, a teacher model initialized with pre-trained weights for the ImageNet dataset \cite{imageNetDataset} was used. Students were distilled using the CUB-200-2011 dataset \cite{cub200Dataset}. The ImageNet dataset contains 1,281,167 training images, 50,000 validation images and 100,000 test images belonging to 1,000 general object classes. The CUB-200-2011 dataset contains 11,768 images of 200 bird species.

\subsection{Models and hyperparameters} 
\textbf{CIFAR-10:} Teacher models are WideResNet-(40,2) and the student models are WideResNet-(16,1). For the VID distillation, the value for $\lambda$ was set to 100. Learning rate was 0.05 at the beginning and was reduced to 0.01 and 0.002 at 150\textsuperscript{th} and 200\textsuperscript{th} epochs respectively. Stochastic Gradient Descent with a weight decay=0.0005 and momentum=0.9 with Nesterov momentum enabled was used as the optimizer. We choose three intermediate layers; the outputs of the second, third and the fourth convolutional blocks of both the teacher and student models. The function $\mu(\cdot)$ for each layer is parameterized using a sequential model with three convolutional layers, ReLU activations and batch normalization in between the layers. A similar architecture and a training setup was used for the baseline (BAS, no distillation) and the RID models. In case of the RID models, the filters $f_s(\cdot)$ and $f_t(\cdot)$ were parameterized using 2-layer convolutional network with a batch normalization layer in the middle. The classification head $g_t(\cdot)$ is a linear layer. We set $n_w=30, q=30, r=1/4$ and the total number of epochs $n+n_w=300$. Teacher, Baseline and VID models are trained for 300 epochs. In both cases of VID and RID, the independent parameter vector $\sigma$ has a dimension equal to the number of channels in the outputs of functions $\mu, f_s$ or $f_t$. All the training was carried out on a computer with an AMD Ryzen Threadripper PRO 5975WX processor and an Nvidia RTX A4500 graphic card. The average training time per model is around 1.5 hours.

\textbf{CIFAR-100:} Teacher models are WideResNet-(28,10) and the student models are WideResNet-(16,8). For the VID distillation, the value for $\lambda$ was set to 100. Learning rate was 0.1 at the beginning and was multiplied by a factor of 0.2 at 60\textsuperscript{th}, 120\textsuperscript{th} and 160\textsuperscript{th} epochs respectively. Stochastic Gradient Descent with a weight decay=0.0005 and momentum=0.9 with Nesterov momentum enabled was used as the optimizer. We choose three intermediate layers; the outputs of the second, third and the fourth convolutional blocks of both the teacher and student models. The function $\mu(\cdot)$ for each layer is parameterized using a sequential model with three convolutional layers, ReLU activations and batch normalization in between the layers. A similar architecture and a training setup was used for the baseline (BAS, no distillation) and the RID models. In case of the RID models, the filters $f_s(\cdot)$ and $f_t(\cdot)$ were parameterized using 2-layer convolutional network with a batch normalization layer in the middle. The classification head $g_t(\cdot)$ is a linear layer. We set $n_w=5, q=5, r=1/5$ and the total number of epochs $n+n_w=250$. Teacher, Baseline and VID models are trained for 250 epochs. In both cases of VID and RID, the independent parameter vector $\sigma$ has a dimension equal to the number of channels in the outputs of functions $\mu, f_s$ or $f_t$. All the training was carried out on a computer with an AMD EPYC 7763 64-Core processor and an Nvidia RTX A6000 Ada graphic card. The average training time per model is around 3.75 hours.

\textbf{ImageNet $\rightarrow$ CUB-200-2011:} The teacher is a ResNet-34 model initialized with weights pre-trained on ImageNet (from Torchvision). The students are ResNet-18 models. For the VID distillation, the value for $\lambda$ was set to 100. Learning rate was 0.1 at the beginning and was multiplied by a factor of 0.2 at 150\textsuperscript{th} and 200\textsuperscript{th} epochs respectively. Stochastic Gradient Descent with a weight decay=0.0005 and momentum=0.9 with Nesterov momentum enabled was used as the optimizer. We choose two intermediate layers; the outputs of the third and the fourth convolutional blocks of both the teacher and student models. The function $\mu(\cdot)$ for each layer is parameterized using a sequential model with two convolutional layers, ReLU activations and batch normalization in between the layers. A similar architecture and a training setup was used for the baseline (BAS, no distillation) and the RID models. In case of the RID models, the filters $f_s(\cdot)$ and $f_t(\cdot)$ were parameterized using 2-layer convolutional network with a batch normalization layer in the middle. The classification head $g_t(\cdot)$ is a linear layer. We set $n_w=20, q=20, r=1/5$ and the total number of epochs $n+n_w=250$. Teacher, Baseline and VID models are trained for 250 epochs. In both cases of VID and RID, the independent parameter vector $\sigma$ has a dimension equal to the number of channels in the outputs of functions $\mu, f_s$ or $f_t$. All the training was carried out on a computer with an AMD Ryzen Threadripper PRO 5975WX processor and an Nvidia RTX A4500 graphic card. The average training time per model is around 7 hours.


\subsection{PID computation} We compute the PID components of the joint information of innermost distilled layers $\mi{Y}{T,S}$, using the framework proposed in \cite{liangPIDCompute} as follows:
\begin{enumerate}[itemsep=0pt, topsep=0pt]
    \item Representations are individually flattened
    \item Compute $n_\text{PCA}$-component PCA on each set of representations
    \item Cluster representations into $n_{C}$ clusters to discretize
    \item Compute the joint distribution $p(Y,T,S)$
    \item Compute PID components using the joint distribution
\end{enumerate}
For CIFAR-10, we use $n_\text{PCA}=n_C=10$. For CIFAR-100, we use $n_\text{PCA}=5$ and $n_C=3$. For the transfer learning setting with ImageNet and CUB-200-2011 the computation was infeasible due to the intermediate representation dimensionality being extremely large.

\subsection{Results}
Table \ref{tab_cifar100Results} presents the accuracies of each framework over CIFAR-100 dataset. Table \ref{tab_transferLearningResults} presents the results of the transfer learning experiment conducted using a teacher trained on ImageNet and students distilled and evaluated over the CUB-200-2011 dataset. In each scenario, RID exhibits more resilience to the untrained teacher compared to VID.

\begin{minipage}{.45\textwidth}
    \centering
    \captionof{table}{Accuracy of each framework on CIFAR-100 dataset when the teacher is trained vs. not trained. A subset of the training data (100 samples per class) was used for the distillation.}
    \begin{tabular}{ccc}
        \toprule
        \textbf{Framework} & \textbf{Trained} & \textbf{Untrained}\\
        \midrule
        RID & 50\% & 41\% \\
        VID & 67\% & 20\% \\
        BAS & 43\% & 43\% \\
        \bottomrule
    \end{tabular}
    \label{tab_cifar100Results}
\end{minipage}
\hspace{0.05\textwidth}
\begin{minipage}{.45\textwidth}
    \centering
    \captionof{table}{Accuracy of each framework in the transfer learning setting (ImageNet teacher $\rightarrow$ CUB-200-2011 student) when the teacher is trained vs. not trained.}
    \begin{tabular}{ccc}
        \toprule
        \textbf{Framework} & \textbf{Trained} & \textbf{Untrained}\\
        \midrule
        RID & 65\% & 33\% \\
        VID & 70\% & 11\% \\
        BAS & 36\% & 36\% \\
        \bottomrule
    \end{tabular}
    \label{tab_transferLearningResults}
\end{minipage}
\end{document}